\documentclass{llncs}

\usepackage{amssymb,amsmath,stmaryrd}
\usepackage{graphicx,color}
\usepackage{xfrac}
\usepackage{hyperref}
\usepackage[]{algorithm2e}
\usepackage[calc]{adjustbox}
\usepackage{wrapfig}
\usepackage{multirow}
\usepackage{tabularx}

\allowdisplaybreaks

\usepackage[appendix=append,bibliography=common]{apxproof}

\newtheoremrep{theorem}{Theorem}[section]
\newtheoremrep{lemma}[theorem]{Lemma}
\newtheoremrep{definition}[theorem]{Definition}
\newtheoremrep{proposition}[theorem]{Proposition}
\newtheoremrep{corollary}[theorem]{Corollary}
\newtheoremrep{example}[theorem]{Example}

\newlength{\strutheight}
\usepackage[colorinlistoftodos,textsize=tiny,color=orange!70%
,disable%
]{todonotes}
\usepackage{tikz}\usetikzlibrary{calc,matrix,arrows,shapes,automata,backgrounds,petri,decorations.pathreplacing}

\newcommand{\N}{\ensuremath{\mathbb{N}}}

\renewcommand{\epsilon}{\varepsilon}

\newcommand{\trans}{\mathsf{t}}

\DeclareMathOperator*{\argmax}{arg\,max}
\newcommand\E{\mathcal{E}}
\newcommand\EPath{\mathcal{E}\!Path}
\newcommand\Tail{T}

\makeatletter
\def\BState{\State\hskip-\ALG@thistlm}
\makeatother

\makeatletter
\newcommand{\subalign}[1]{%
  \vcenter{%
    \Let@ \restore@math@cr \default@tag
    \baselineskip\fontdimen10 \scriptfont\tw@
    \advance\baselineskip\fontdimen12 \scriptfont\tw@
    \lineskip\thr@@\fontdimen8 \scriptfont\thr@@
    \lineskiplimit\lineskip
    \ialign{\hfil$\m@th\scriptstyle##$&$\m@th\scriptstyle{}##$\hfil\crcr
      #1\crcr
    }%
  }%
}
\makeatother

\author{Rebecca Bernemann\inst{1} \and Barbara K\"onig\inst{1}\and
  Matthias Schaffeld\inst{1} \and Torben Weis\inst{1}}
\authorrunning{R. Bernemann et al.}
\institute{Universit\"at Duisburg-Essen, Duisburg, Germany}
\title{Probabilistic Systems with Hidden State and Unobservable
  Transitions}

\begin{document}

\maketitle

\begin{abstract}
  We consider probabilistic systems with hidden state and unobservable
  transitions, an extension of Hidden Markov Models (HMMs) that in
  particular admits unobservable $\epsilon$-transitions (also called
  null transitions), allowing state changes of which the observer is
  unaware. Due to the presence of $\epsilon$-loops this additional
  feature complicates the theory and requires to
  carefully set up the corresponding probability space and random
  variables. In particular we present an algorithm for determining the
  most probable explanation given an observation (a generalization of
  the Viterbi algorithm for HMMs) and a method for parameter learning
  that adapts the probabilities of a given model based on an
  observation (a generalization of the Baum-Welch algorithm). The
  latter algorithm guarantees that the given observation has a higher
  (or equal) probability after adjustment of the parameters and its
  correctness can be derived directly from the so-called EM algorithm.
\end{abstract}

\section{Introduction}
\label{sec:introduction}

There are many practical applications that involve the observation of
a probabilistic system with hidden state, where the aim is to infer
properties about the state of the system only from the observations
that are available.

In particular we are motivated by the following scenario: imagine a
building equipped with sensors that are triggered when a person walks
past. However, these sensors might produce both false positives
(nobody walked past, but the sensor sends a signal) and false
negatives (somebody was present, but did not trigger the sensor). This
can be modelled by a probabilistic transition system which has both
observable symbols and $\epsilon$-transitions (also referred to as
null transitions), corresponding to false negatives. Now assume that
there are three rooms, the bedroom (B), the corridor (C) and the
kitchen (K), all of them connected through C and equipped with
sensors. Sensors B, K trigger, but not C. However, in order to
reach the kitchen from the bedroom, the person should have passed the
corridor! Hence our analysis should tell us that the most likely
explanation for the observation is indeed the sequence B, C, K.

While here this reasoning is straightforward, 
it may become
increasingly more complex with additional missing sensor data and
multiple possible paths.

In order to make matters more concrete, consider the following
system depicting our motivational example.
The start state is $s_0$ and from each state we label the
transitions with symbols and probabilities. For instance, from $s_0$
there is a probability of $0.1$ of going to C with an
unobservable $\epsilon$-transition.

Now we detect the observation: b,k. What happened? In fact there are
several possible paths, the most probable being $s_0$, B, C, K (first
transition b, second one $\epsilon$ and the third k) whose probability
is
\[\delta(s_0)(\text{b}, \text{B}) \cdot \delta(\text{B})(\epsilon, \text{C}) 
\cdot \delta(\text{C})(\text{k}, \text{K}) = 0.4 \cdot 0.3 \cdot 0.5 = 0.06\]

\todo{M:\\  
        1. $s_0$ -b-> B -$\epsilon$-> C -k-> K \\
        2. $s_0$ -$\epsilon$-> B -$\epsilon$-> C -b-> B -$\epsilon$-> C -k-> K \\
        3. $s_0$ -$\epsilon$-> C -b-> B -$\epsilon$-> C -k-> K\\
        4. $s_0$ -$\epsilon$-> K -$\epsilon$-> C -b-> B -$\epsilon$->
        C -k-> K \\
      B: point taken ;-) I wrote ``several possible paths''}

\begin{wrapfigure}[7]{r}{0.42\linewidth}
  \vspace{-1.35cm}
\begin{center}
  \begin{tikzpicture}
    \node[state]             (C) {C};
    \node[state, left=of C, xshift=-2mm] (B) {B};
    \node[state, right=of C, xshift=2mm] (K) {K};
    \node[state, above=of C, yshift=-2mm] (S) {$s_0$};
    
    \draw[every loop]
    (B) edge[auto=left, draw=white, bend left=10]  node {$\subalign{\text{c} \  &0.7 \\
        \epsilon \  &0.3}$} (C)
    (B) edge[auto=left, bend left=10]  node {} (C)

    (C) edge[auto=left, draw=white, bend left=10]  node {$\subalign{\text{b} \  &0.3 \\
        \epsilon \  &0.1}$} (B)
    (C) edge[auto=left, bend left=10]  node {} (B)

    (C) edge[auto=left, draw=white, bend left=10]  node {$\subalign{\text{k} \  &0.5 \\
        \epsilon \  &0.1}$} (K)
    (C) edge[auto=left, bend left=10] node {} (K)

    (K) edge[auto=left, draw=white, bend left=10]  node {$\subalign{\text{c} \  &0.8 \\
        \epsilon \  &0.2}$} (C)
    (K) edge[auto=left, bend left=10] node {} (C)

    (S) edge[auto=right, bend right=10] node {$\subalign{\text{b} \  &0.4 \\
        \epsilon \  &0.05}$} (B)
    (S) edge[auto=left, bend left=10] node {$\subalign{\text{k} \  &0.2 \\
        \epsilon \  &0.05}$} (K)
    (S) edge[auto=right] node[yshift=4pt] {$\subalign{\text{c} \  &0.2 \\
        \epsilon \  &0.1}$} (C);

    \draw[->, black] (0,2.35) -- (S) node[pos=.35, xshift=2mm]{};
  \end{tikzpicture}
\end{center}  
\end{wrapfigure}
The question is how to efficiently determine the most likely path and
its probability.

A second issue is how to learn the probabilities that label the
transitions. 
Assume 
the basic structure of the system is known, in particular the number
of states, but 
the parameters, i.e., the probabilities are not. Now we observe the
system and want to estimate its parameters.

Of course, such systems have been extensively studied under the name
of Hidden Markov Models (HMMs)
\cite{r:tutorial-hidden-markov-models,wv:hmm-methods-protocols}. 
Unobservable transitions, also known as null transitions or
$\epsilon$-transitions, i.e. HMMs that may change state without the
observer being aware of it, have been proposed in the literature,
especially in the context of speech recognition
\cite{bjm:max-likelihood-speech,j:statistical-methods}.  However, to
the best of our knowledge, there is no theory of systems that allows
for $\epsilon$-loops. Related work only treats specific HMMs, where
either loops can not appear altogether (so-called left-to-right HMMs)
\cite{otdfd:learning-scripts} or $\epsilon$-loops are forbidden
\cite{b:maximization-technique,bjm:max-likelihood-speech}.  Also
\cite{j:statistical-methods} does not allow $\epsilon$-loops in the
context of learning. It does describe an algorithm to eliminate
$\epsilon$-transitions, which can increase the size of the model and
alters the system such that probabilities of $\epsilon$-transitions
can not be learned directly.



On the other hand $\epsilon$-loops can occur naturally in
applications, e.g., in the application described above that we have in
mind, and there are no easy work\-arounds. We can only speculate why the
generalization has not been made, but observe that for systems with
$\epsilon$-loops%
\todo{M: Bis jetzt immer $\epsilon$-loops und kommt auch nur hier
  vor. Ist das eine beabsichtigte unterscheidung? \\ B: Ist dasselbe, ich
habe jetzt aber ``loops'' geschrieben.}
the theory (in particular for parameter learning) is more complex since
the Baum-Welch algorithm has to be fundamentally adapted to deal with
this scenario. In particular the usual forward-backward algorithm
\cite{r:tutorial-hidden-markov-models} for parameter estimation can
not be used directly, but has to be generalized.
%

%
Given an observation, we propose to compute the conditional
expectation of passing each transition (remember that the state and so
the transitions are hidden) and this is turned into a probability
function (by normalizing), which is the new parameter estimate.  The
guarantee provided by this estimate is that the observation sequence
becomes more probable given the new parameters. This result can be
derived from the so-called expectation-maximization (EM) algorithm
that gives us a general framework for such results.

Our contributions are:

\noindent$\triangleright$ We rephrase the theory behind HMMs with
$\epsilon$-transitions, being precise about the used probability space
and random variables.

\noindent$\triangleright$ We extend the theory to systems with
$\epsilon$-transitions, a very natural extension for such systems and
indispensable for our application, which complicates the formalization
and the algorithms. In particular, we 
have to handle $\epsilon$-loops 
which we deal with by setting up fixpoint equations.

\noindent$\triangleright$ We spell out explicitly why parameter
learning based on the EM algorithm works in this setting.

While HMMs have been known for some time, we feel that, due to the
current large interest in learning approaches (e.g., machine
learning or the L* algorithm \cite{a:l-star}), it makes sense to
revive the theory and close existing gaps in the literature.

\section{Preliminaries}
\label{sec:preliminaries}

\paragraph*{Probability theory.}

We recapitulate the basics of discrete probability theory.

A \emph{probability space} $(\Omega,P)$ consists of a (countable)
\emph{sample space} $\Omega$ and a probability function
$P\colon \Omega\to [0,1]$ such that $\sum_{\omega\in\Omega} P(\omega) =
1$. Given a set $\Omega$ we denote by $\mathcal{D}(\Omega)$ the set of
all \emph{probability functions} on $\Omega$.

A \emph{random variable} for a probability space $(\Omega,P)$ is a
function $X\colon \Omega\to V$. We assume a special random variable
$Z\colon\Omega\to\Omega$, which is the identity. For $v\in V$ we
denote by $P(X=v) = \sum_{X(\omega)=v} P(\omega)$ the probability that
$X$ has value $v$. Given two random variables
$X_i\colon \Omega\to V_i$, $i=1,2$, the conditional probability that
$X_1$ takes value $v_1$, under the condition that $X_2$ takes value
$v_2$, is
\[ P(X_1=v_1\mid X_2=v_2) = \frac{P(X_1=v_1\land
    X_2=v_2)}{P(X_2=v_2)}, \]
provided that $P(X_2=v_2) > 0$.

For a random variable $X\colon\Omega\to \mathbb{R}$ (where we might
enrich the real numbers with $-\infty$), we define its
\emph{expectation} as
$E[X] = \sum_{\omega\in\Omega} X(\omega) \cdot P(\omega)$. Given
another random variable $Y\colon \Omega\to V$, conditional expectation
is defined, for $v\in V$, as
\[ E[X\mid Y=v] = \sum_{\omega\in\Omega} X(\omega) \cdot P(Z=\omega\mid Y=v) . \]

\paragraph*{Hidden Markov Models.}

We are working with HMMs with $\epsilon$-transitions (also called null
transitions in the literature
\cite{bjm:max-likelihood-speech,j:statistical-methods}).  \todo{M: Der
  Einschub steht hier zum dritten mal: Abstract, Einleitung (referred
  to...) und hier \\ B: Im Abstract sehe ich es jetzt nicht (da habe
  ich auch mal was geändert). Ansonsten ist es mMn ok, das zweimal zu
  sagen.}  In particular we put observations on the transitions,
rather than on the states \cite{r:tutorial-hidden-markov-models}. This
is a standard variant for HMMs and is for instance done in
\cite{bjm:max-likelihood-speech,j:statistical-methods,dde:pa-hmm}. Since
we use $\epsilon$-transitions, we need not assume an initial
probability function, but instead later fix a start state $s_0$.

\begin{definition}[HMM with $\epsilon$-transitions]
  An HMM with $\epsilon$-transitions is a three-tuple
  $(S,\Sigma,\delta)$, consisting of a finite state space $S$, an
  alphabet $\Sigma$ and a transition function
  $\delta\colon S \rightarrow \mathcal{D}((\Sigma \cup \{\epsilon\})
  \times S)$.
  The set of transitions is defined as
  $T_\delta = \{(s, a, s') \mid s, s' \in S, a \in \Sigma \cup
  \{\epsilon\}, \delta(s)(a,s') > 0\}$. In addition
  $T_\delta^s = T_\delta\cap (\{s\}\times (\Sigma\cup \{\epsilon\}) \times
  S)$.
\end{definition}


\section{Probability Space and Random Variables}
\label{sec:probability-space-random-variables}

The probability space of observation sequences contains alternating
sequences of states and observation symbols (or $\epsilon$) and is
dependent on $n$, denoting the number of observable symbols (from
$\Sigma$). We fix a start state $s_0\in S$ and restrict the possible
sequences to those where the second to last element is contained in
$\Sigma$ and is hence observable. This is needed to make sure that the
probabilities in fact sum up to~$1$.
\[\Omega^n_{s_0} = s_0 ((\epsilon S)^* \Sigma S)^n\]
The probability for an element from the probability space
$\widetilde{z} \in \Omega^n_{s_0}$ can be calculated by multiplying
the corresponding transition probabilities. Whenever
$\widetilde{z} = s_0 a_1 s_1 a_2 \dots a_m s_m \in \Omega^n_{s_0}$
where $s_i \in S$ and $a_i \in \Sigma \cup \{\epsilon\}$ we define:
\todo{$\in \Omega^n$ oder $\in \Omega^n_{s_0}$ \\ B: ?? Wir sagen
  weiter oben, dass $\widetilde{z} \in \Omega^n_{s_0}$.}
$P^n_{s_0}(\widetilde{z}) = \prod_{i=1}^m \delta(s_{i-1})(a_i,
s_i)$. Furthermore $P^0_{s_0}(s_0)=1$.

Note that due to the presence of $\epsilon$-transitions we have to take
care to set up a probability space where the probabilities add up
to $1$. An alternative could be to use the solution of
\cite{j:statistical-methods} and to distinguish a final state, which
is however inconvenient for some applications. We continue by showing
that the probability space is well-defined under some mild conditions.
These conditions have the additional benefit that the fixpoint
equations become contractive (after a number of iterations)
and hence have unique solutions (for more details see
Appendix~\ref{sec:contractivity}).

\begin{propositionrep}
  \label{prop:prob-space-welldefined}
  Assume that for each state $s\in S$ there is an outgoing path of
  non-zero probability that contains a symbol in $\Sigma$. Then the
  probability space is well-defined, in particular
  $\sum_{\widetilde{z}\in\Omega^n_{s_0}} P^n_{s_0}(\widetilde{z}) =
  1$.
\end{propositionrep}

\begin{proof}
  Given a sequence
  $\widetilde{z} = s_0 a_1 s_1 a_2 \dots a_m s_m = s_0 a_1 \widetilde{z}_1$ we
  observe that
  \begin{align*}
    P^n_{s_0}(\widetilde{z}) &= \delta(s_0)(a_1,s_1)\cdot
    P^n_{s_1}(\widetilde{z}_1) &\text{if $a_1=\epsilon$} \\
    P^n_{s_0}(\widetilde{z}) &= \delta(s_0)(a_1,s_1)\cdot
    P^{n-1}_{s_1}(\widetilde{z}_1) &\text{if $a_1\in \Sigma$}
  \end{align*}
  We abbreviate
  $S^n_{s_0} = \sum_{\widetilde{z}\in\Omega^n_{s_0}}
  P^n_{s_0}(\widetilde{z})$. Since $\Omega^0_{s_0} = \{s_0\}$, it is
  easy to see that $S^0_{s_0} = 1$ and if $n\ge 1$:
  \begin{align*}
    S^n_{s_0} = \sum_{\widetilde{z}\in\Omega^n_{s_0}} P^n_{s_0}(\widetilde{z})
    &= \sum_{s_1\in S} \sum_{\widetilde{z}_1\in\Omega^n_{s_1}}
    \delta(s_0)(\epsilon,s_1)\cdot P^n_{s_1}(\widetilde{z}_1)
    \mathop{+} \\
    &\sum_{a_1\in\Sigma, s_1\in S}
    \sum_{\widetilde{z}_1\in\Omega^{n-1}_{s_1}}
    \delta(s_0)(a_1,s_1)\cdot P^{n-1}_{s_1}(\widetilde{z}_1) \\
    &= \sum_{s_1\in S} \delta(s_0)(\epsilon,s_1)\cdot
    \sum_{\widetilde{z}_1\in\Omega^n_{s_1}} P^n_{s_1}(\widetilde{z}_1)
    \mathop{+} \\
    &\sum_{a_1\in\Sigma, s_1\in S} \delta(s_0)(a_1,s_1)\cdot
    \sum_{\widetilde{z}_1\in \Omega^{n-1}_{s_1}} P^{n-1}_{s_1}(\widetilde{z}_1) \\
    &= \sum_{s_1\in S} \delta(s_0)(\epsilon,s_1)\cdot S^n_{s_1}
    \mathop{+} \sum_{a_1\in\Sigma, s_1\in S} \delta(s_0)(a_1,s_1)\cdot
    S^{n-1}_{s_1}
  \end{align*}
  Since the probabilities of all outgoing transitions of a state sum
  up to $1$, we observe that $S^n_{s} = 1$ for all $n,s\in S$ is a
  solution to this system of fixpoint equations. Since we assume that
  each state will eventually reach an observation symbol in $\Sigma$
  with non-zero probability, the corresponding fixpoint function is
  contractive after some iterations, since the probability to stay
  with index $n$ is strictly less than $1$ after at most $|S|$ steps
  (for more details see Appendix~\ref{sec:contractivity}).

  This implies that the fixpoint is unique and hence the statement of
  the proposition follows.  \qed
\end{proof}

We will in the following assume that the requirement of
Prop.~\ref{prop:prob-space-welldefined} holds. Otherwise there
might be states that can never reach an observation symbol, for which
the probability of all outgoing paths is $0$.

Given this probability space, we define some required random
variables.

\vspace*{0.1cm}
\noindent\begin{tabularx}{\textwidth}{|l|X|}
  \hline
  \multicolumn{2}{|c|}{Random Variables} \\
  \hline
  $Z: \Omega^n_{s_0} \rightarrow \Omega^n_{s_0}$ & Identity on $\Omega^n_{s_0}$\\
  $Y: \Omega^n_{s_0} \rightarrow \Sigma^n$ \qquad& Projection to
  observable symbols (removal of $\epsilon$'s and states) \\
  $L: \Omega^n_{s_0} \rightarrow S$  & Last state of a given observation sequence \\ 
  $X_\trans: \Omega^n_{s_0} \rightarrow \N_0$  & Number of times a
  transition $\trans = (s, a, s')$ occurs in a sequence \\
  \hline
\end{tabularx}
\vspace*{0.1cm}


We omit the indices $n,s_0$ if they are clear from the context: if we
write $P(\widetilde{z})$ or $P(Z=\widetilde{z})$ we work in the
probability space $\Omega^n_{s_0}$ and mean the probability function
$P^n_{s_0}$, where $n=|Y(\widetilde{z})|$ and $s_0$ is the first
element of $\widetilde{z}$. And if we write
$P_{s_0} (Y=\widetilde{y})$, the value $n$ is understood to be
$|\widetilde{y}|$. We do the same for expectations.

\section{Finding the Best Explanation for an Observation}
\label{sec:best-explanation}

As a warmup we will describe a method for finding the best
explanation, given an observation sequence. More concretely, an
observation sequence $\widetilde{y}\in\Sigma^*$ is given and it is our
aim to compute the most probable sequence of states and its
probability. For standard
HMMs there is a well-known algorithm for this task: the Viterbi
algorithm
\cite{v:viterbi-algorithm,f:viterbi-algorithm,r:tutorial-hidden-markov-models}. Instead
of enumerating all paths and checking which one is most probable, it
uses intermediate results, by computing step-by-step the most probable
path ending at a given state $s$, for each prefix\todo{B: prefix or
  suffix?} of the observation sequence $\widetilde{y}$.

We now adapt the Viterbi algorithm, taking $\epsilon$-transitions into
account. While in the standard case it is straightforward to obtain
the likeliest path in the case of a single observation symbol, in our
case the path might have taken an arbitrary number of
$\epsilon$-transitions in between.
Remember that the probability space is set up in such a way that the
last transition in every sequence that we consider is always
observable, which is no restriction, since there is always some
explanation with maximal probability that satisfies this condition.



\begin{propositionrep}[Maximal probability for one observation]
  \label{prop:Epsilon}
  Let $(S, \Sigma, \delta)$ be an HMM and $a\in\Sigma$ be an
  observation.  With $\E^a_{s_0, s}$, for $s_0,s\in S$ we denote the
  probability for the most likely path in $\Omega^1_{s_0}$, starting
  in state $s_0$ and ending in state $s$, where $a$ is the
  observation. Then we have:
  \begin{align*}
    \E^a_{s_0, s} &= \max_{\substack{\widetilde{z} \in \Omega^1_{s_0}\\
        L(\widetilde{z}) = s}}
      P(Z = \widetilde{z} \land Y = a)
      = \max \Big( \delta(s_0)(a, s)\, , \max_{s' \in S}
      \delta(s_0)(\epsilon, s') \cdot \E^a_{s', s}
      \Big)
  \end{align*}
\end{propositionrep}

\begin{proof}
  \begin{align*}
    \E^a_{s_0, s} &= \max_{\substack{\widetilde{z} \in \Omega_{s_0}^1 \\
        L(\widetilde{z}) = s}} P^1(Z = \widetilde{z} \land Y = a) 
    = \max_{\substack{\widetilde{z} \in \\
        s_0(\epsilon S)^* a s}}
    P^1(Z = \widetilde{z} \land Y = a)\\
    &= \max \Big( \underbrace{ P^1(Z = s_0as \land Y = a)
    }_{\delta(s_0)(a, s)} \, ,
    \max_{\substack{\widetilde{z}_1 \in \\
        S(\epsilon S)^* a s}}
    P^1(Z = s_0 \epsilon \widetilde{z}_1 \land Y = a) \Big) \\
    &= \max \Big( \delta(s_0)(a, s) \, ,\ \max_{s' \in S}
    \max_{\substack{\widetilde{z}_1 \in \\ s'(\epsilon
        S)^*a s}} \delta(s_0)(\epsilon, s') \cdot
    P^1(Z = \widetilde{z}_1 \land Y = a) \Big) \\
    &= \max \Big( \delta(s_0)(a, s) \, ,\ \max_{s' \in S}
    \delta(s_0)(\epsilon, s') \cdot
    \underbrace{ \max_{\substack{\widetilde{z}_1 \in \\
          s'(\epsilon S)^*a s}}
      P^1(Z = \widetilde{z}_1 \land Y = a) \Big)
    }_{\E^a_{s', s}} \\
    &= \max \Big( \delta(s_0)(a, s)\, , \max_{s' \in
      S} \delta(s_0)(\epsilon, s')
    \cdot \E^a_{s', s} \Big)
  \end{align*}
  \qed
\end{proof}

The equation of Prop.~\ref{prop:Epsilon} has a unique fixpoint
due to the requirement that from every state there is a path of
non-zero probability that contains an observation.  In order to
compute $\E^a_{s_0, s}$ one could hence perform fixpoint iteration or
use an external solver. In fact, the computation is simplified in this
case since among the paths with the highest probability there is
always one that does not contain duplicate states (apart from the
final state $s$). By equipping the computation with an extra parameter
$S_0\subseteq S$ (the set of states that can still be visited), we can
easily ensure termination, even in the presence of $\epsilon$-loops,
and the equation becomes the following, where $\E^a_{s_0, s} =
\E^a_{s_0, s}(S)$. 
\[ \E^a_{s_0, s}(S_0) = \max \Big( \delta(s_0)(a, s)\, ,
  \max_{s' \in S_0\backslash \{s_0\}} \delta(s_0)(\epsilon,
  s') \cdot \E^a_{s', s}(S_0\backslash\{s_0\})
  \Big) \] 

We can now address the task of computing the maximal probability for a
longer sequence of observations.  For this purpose, we extend the
established Viterbi algorithm \cite{v:viterbi-algorithm}.  Here, the
probability for the likeliest path that results in a given observation
is computed inductively and is based on
Prop.~\ref{prop:Epsilon}.

\begin{propositionrep}[Maximal probability for observation sequence]
  Let $(S, \Sigma, \delta)$ be an HMM and let
  $\widetilde{y} = a_1 \dots a_n = \widetilde{y}_1a_n$ be an observation
  sequence. Then $V_{s_0,s}^{\widetilde{y}}$ denotes the maximum
  probability of observing $\widetilde{y}$ and ending in state $s$,
  more formally
  \[ V_{s_0,s}^{\widetilde{y}} =
    \max_{\substack{\widetilde{z}\in\Omega^n_{s_0}\\
        L(\widetilde{z})=s}} P(Z = \widetilde{z} \land Y =
    \widetilde{y}) \] For $n=0$ we have $V_{s_0,s}^{\epsilon}=1$ if
  $s = s_0$ and $0$ otherwise. For $n>0$:
  \begin{align*}
    V_{s_0,s}^{\widetilde{y}} &= \max_{s' \in S} \;
    V_{s_0,s'}^{\widetilde{y}_1} \cdot \E^{a_n}_{s', s} 
  \end{align*}
\end{propositionrep}

\begin{proof}
\begin{align*}
  &V_{s_0,s}^{\widetilde{y}} = V_{s_0,s}^{a_1 \dots a_n} \\
  &=
  \max_{\substack{\widetilde{z}\in\Omega^n_{s_0}\\
      L(\widetilde{z})=s}} \;
  P(Z = \widetilde{z} \land Y = a_1 \dots a_n) \\
  &= \max_{s' \in S} \;
  \max_{\substack{\widetilde{z}_1\in\Omega^{n-1}_{s_0}\\L(\widetilde{z}_1)
      = s'}}
  \max_{\substack{\widetilde{z}_2 \in \Omega^{1}_{s'}\\
      L(\widetilde{z}_2) = s}} \; P(Z = \widetilde{z}_1 \land Y =
  a_1\dots a_{n-1}) \, \cdot
  P(Z = \widetilde{z}_2 \land Y = a_n)\\
  &= \max_{s' \in S} \; \underbrace{
    \max_{\substack{\widetilde{z}_1\in\Omega^{n-1}_{s_0}\\L(\widetilde{z}_1)
        = s'}} \; P(Z = \widetilde{z}_1 \land Y = a_1\dots a_{n-1})
  }_{V_{s_0,s'}^{a_1 \dots a_{n-1}}} \, \cdot \underbrace{
    \max_{\substack{\widetilde{z}_2 \in \Omega^{1}_{s'}\\
        L(\widetilde{z}_2) = s}} P(Z = \widetilde{z}_2 \land Y = a_n)
  }_{\E^{a_n}_{s', s}}\\
  &= \max_{s' \in S} \; V_{s_0,s'}^{a_1 \dots a_{n-1}}\cdot
  \E^{a_n}_{s', s} = \max_{s' \in S} \; V_{s_0,s'}^{\widetilde{y}_1}\cdot
  \E^{a_n}_{s', s}
\end{align*}
  \qed
\end{proof}

In order to obtain the best explanation starting at $s_0$, regardless
of its final state, we still have to take the maximum
$\max_{s\in S} V_{s_0,s}^{\widetilde{y}}$. If we are instead
interested in the conditional probability, i.e.,
$\max_{\widetilde{z}\in\Omega^n_{s_0}} P(Z = \widetilde{z} \mid Y =
\widetilde{y})$, it can be obtained from this maximum by dividing by
$P_{s_0}(Y=\widetilde{y})$.

Since the computation of the most likely path is almost identical to
the computation of the highest probability, we elaborate on this only
in the appendix.

\begin{toappendix}
In order to obtain the most likely path resulting in a given observation,
we make use of the calculated highest probabilities. We do
not only expect a sequence of states as explanation, but also the
intermediate symbols or $\epsilon$'s used when transitioning from
state to state. This is necessary because a possible
$\epsilon$-transition can implicitly occur in an observation, creating
an ambiguity problem when working out which exact transitions where
taken at what time in the state sequence.


By unravelling the fixpoint equation of
Prop.~\ref{prop:Epsilon}, we obtain the following construction,
where $\EPath^a_{s_0, s}\in \Omega^1_{s_0}$ denotes the likeliest state
sequence for a path starting in state $s_0$ and ending in state $s$
that produces the observation $a$ in its last transition. This is
feasible whenever $\E^a_{s_0,s} > 0$.

\[
  \EPath^a_{s_0, s} = \left\{
    \begin{array}{ll}
      s_0as & \mbox{if $\E_{s_0,s}^y = \delta(s_0)(a,s)$} \\
      s_0\epsilon \EPath^a_{s', s} & \mbox{otherwise, with
        $s' = \argmax_{s'\in S}
        \delta(s_0)(\epsilon,s')\cdot \E_{s',s}^y$}
    \end{array}\right.
\]

Note that in the following the operator $\circ$ denotes
concatenation and the tail function $\Tail$ removes the first
element of a given input sequence and returns the rest.

\begin{proposition} [Likeliest state sequence for observation sequence]
  Given an HMM $(S, \Sigma, \delta)$ and an observation sequence
  $\widetilde{y} = a_1 \dots a_n$. Whenever
  $V_{s_0,s}^{\widetilde{y}} > 0$, the term
  $\psi_{s_0,s}^{\widetilde{y}}$ denotes the likeliest state sequence
  starting in state $s_0$, ending in $s$, that explains
  $\widetilde{y}$, in particular
  \[ \psi_{s_0,s}^{\widetilde{y}} = \argmax_{\substack{\widetilde{z}\in
        \Omega^n_{s_0} \\ L(\widetilde{z})=s}} P(Z = \widetilde{z}
    \land Y = \widetilde{y}). \] It holds that
  $\psi_{s_0,s_0}^{\epsilon} = s_0$.  Furthermore, whenever
  $\widetilde{y} = \widetilde{y}_1a_n$:
  \begin{align*}
    \psi_{s_0,s}^{\widetilde{y}} &= 
     \psi_{s_0,s'}^{\widetilde{y}_1}\circ \Tail(\EPath^{a_n}_{s',s}),
  \end{align*}
  where
  $s' = \argmax_{s' \in S} \; 
  V_{s_0,s'}^{\widetilde{y}_1}\cdot \E^{a_n}_{s',s}$.
\end{proposition}

\begin{proof}
  \begin{align*}
    &\psi_{s_0,s}^{\widetilde{y}} = \psi_{s_0,s}^{a_1 \dots a_n} = \\
    &=
    \argmax_{\substack{\widetilde{z}\in\Omega^n_{s_0}\\L(\widetilde{z})=s}}
    P(Z = \widetilde{z} \land Y = a_1 \dots a_n) \\
    &= \underbrace{
      \argmax_{\substack{\widetilde{z}_1\in\Omega^{n-1}_{s_0}\\
          L(\widetilde{z}_1) = s'}} P(Z = \widetilde{z}_1 \land Y =
      a_1\dots a_{n-1}) }_{\psi^{a_1\dots a_{n-1}}_{s_0, s'}} \, \circ
    \Tail(\underbrace{
      \argmax_{\substack{\widetilde{z}_2 \in \Omega^{1}_{s'}\\
          L(\widetilde{z}_2)=s}} P(Z = \widetilde{z}_2 \land Y = a_n)
    }_{\EPath_{s',s}^{a_n}})\\
    &= \psi_{s_0,s'}^{a_1 \dots a_{n-1}} \circ \Tail(\EPath^{a_n}_{s', s}) =
    \psi^{\widetilde{y}_1}_{s_0, s'}\circ \Tail(\EPath_{s',s}^{a_n})
  \end{align*}
  where
  $s' = \argmax_{s' \in S} \; V_{s_0,s'}^{a_1 \dots a_{n-1}}\cdot
  \E^{a_n}_{s', s}$ is the state where the maximum is reached. \qed
\end{proof}
\end{toappendix}

\section{Parameter Learning}
\label{sec:parameter-learning}

We now discuss a method for determining the system parameters. We
assume that the structure of the system and initial probabilities are
given, and those probabilities have to be adjusted through observing
output sequences.  This core problem for HMMs is traditionally solved
by the Baum-Welch algorithm \cite{b:maximization-technique}, which is
based on the forward-backward algorithm, but because of
$\epsilon$-transitions and in particular $\epsilon$-loops, it is
necessary to develop a different approach.

\subsection{Conditional Expectation of the Number of Transition
  Traversals}


To adjust the probabilities, we have to solve the following subtask:
Given an HMM with initial state $s_0$, an observation sequence
$\widetilde{y}$ 
and a transition $\trans$, determine the expected value of the number
of traversals of $\trans$, when observing sequence $\widetilde{y}$,
starting from $s_0$.
For each state, we determine these values for all outgoing transitions
and normalize them to obtain probabilities. This gives us new
parameters and we later discuss the guarantees that this approach
provides.

If there are no $\epsilon$-loops, it is sufficient to compute the
probability of crossing a given transition $\trans$ while reading the
$i$-th symbol of the observation sequence and to sum up over all
$i$. This is done with the forward-backward algorithm, determining the
probability of reaching the source state of $\trans$, multiplied with
the probability of $\trans$ and the probability of reading the
remaining observation sequence from the target state. In the present
setup, this has to be adapted, since we may cross $\trans$ several
times while reading the $i$-th symbol.



We want to determine $E_{s_0}[ X_\trans \mid Y=\widetilde{y}]$ or,
equivalently,
$E_{s_0}[ X_\trans \mid Y=\widetilde{y}]\cdot
P_{s_0}(Y=\widetilde{y})$. This is defined if
$P_{s_0}(Y=\widetilde{y}) > 0$, which we assume since the sequence
$\widetilde{y}$ has actually been observed. Note that due to the
nature of our probability space, $\epsilon$-transitions that might be
traversed after the last observation do not count.
We compute the conditional expectation by setting up a suitable
fixpoint equation.


\begin{propositionrep}
  \label{prop:cond-exp}
  Fix an HMM and an observation sequence
  $\widetilde{y} = a_1\dots a_n = a_1\widetilde{y}_1$.  Let
  $\trans = (s,a,s')$ and define
  \[ C^{\widetilde{y}}_{s_0,\trans} = E_{s_0}[X_\trans \mid Y=\widetilde{y}]
    \cdot P_{s_0}(Y=\widetilde{y}). \]

  Then $C^{\epsilon}_{s_0,\trans} = 0$ and the following fixpoint
  equation holds: whenever $a\in\Sigma$
  \begin{align*}
    C^{\widetilde{y}}_{s_0,\trans} &= \sum_{s_1\in S}
    \delta(s_0)(a_1,s_1) \cdot C^{\widetilde{y}_1}_{s_1,\trans} +
    \sum_{s_1\in S} \delta(s_0)(\epsilon,s_1) \cdot
    C^{\widetilde{y}}_{s_1,\trans} \mathop{+} \\
    & \qquad\qquad [s_0=s \land a_1=a]\cdot
    \delta(s)(a,s')\cdot P_{s'}(Y=\widetilde{y}_1)
  \end{align*}
  and whenever $a=\epsilon$ the last summand has to be replaced by
  $[s_0=s]\cdot \delta(s)(\epsilon,s')\cdot P_{s'}(Y=\widetilde{y})$.
  We use the convention that $[b]=1$ if $b$ holds and $[b]=0$
  otherwise.
\end{propositionrep}

\begin{proof}
  Note that by assumption $P_{s_0}(Y=\widetilde{y}) > 0$.
  We compute
  \begin{align*}
    C^{\widetilde{y}}_{s_0,\trans} &= E_{s_0}[X_\trans \mid
    Y=\widetilde{y}]
    \cdot P_{s_0}(Y=\widetilde{y}) \\
    &= \sum_{\widetilde{z}\in\Omega^n_{s_0}}
    X_\trans(\widetilde{z})\cdot
    P(Z=\widetilde{z}\mid Y=\widetilde{y}) \cdot P_{s_0}(Y=\widetilde{y}) \\
    &= \sum_{\widetilde{z}\in\Omega^n_{s_0}}
    X_\trans(\widetilde{z})\cdot
    P(Z=\widetilde{z}\land Y=\widetilde{y}) \\
    &=
    \sum_{\substack{\widetilde{z}\in\Omega^n_{s_0}\\Y(\widetilde{z})=\widetilde{y}}}
    X_\trans(\widetilde{z})\cdot P(Z=\widetilde{z}) \\
    &= \sum_{s_1\in S}
    \sum_{\substack{\widetilde{z}_1\in\Omega^{n-1}_{s_1}\\Y(s_0a_1\widetilde{z}_1)=\widetilde{y}}}
    X_\trans(s_0a_1\widetilde{z}_1)\cdot P(Z=s_0a_1\widetilde{z}_1)
    \mathop{+} \\
    & \qquad\qquad \sum_{s_1\in S} \sum_{\substack{\widetilde{z}_1\in
        \Omega^n_{s_1}\\Y(s_0\epsilon\widetilde{z}_1)=\widetilde{y}}}
    X_\trans(s_0\epsilon\widetilde{z}_1)\cdot P(Z=s_0\epsilon\widetilde{z}_1) \\
    &= \sum_{s_1\in S}
    \sum_{\substack{\widetilde{z}_1\in\Omega^{n-1}_{s_1}\\Y(\widetilde{z}_1)=\widetilde{y}_1}}
    (X_\trans(\widetilde{z}_1)+[s_0=s\land a_1=a \land s_1=s'])
    \mathop{\cdot} \\
    & \qquad\qquad\qquad\delta(s_0)(a_1,s_1)\cdot P(Z=\widetilde{z}_1)
    \mathop{+} \\
    & \qquad\qquad \sum_{s_1\in S}
    \sum_{\substack{\widetilde{z}_1\in\Omega^n_{s_1}\\Y(\widetilde{z}_1)=\widetilde{y}}}
    X_\trans(\widetilde{z}_1)\cdot \delta(s_0)(\epsilon,s_1)\cdot P(Z=\widetilde{z}_1) \\
    &= \sum_{s_1\in S} \delta(s_0)(a_1,s_1)\cdot
    \sum_{\substack{\widetilde{z}_1\in\Omega^{n-1}_{s_1}\\Y(\widetilde{z}_1)=\widetilde{y}_1}}
    X_\trans(\widetilde{z}_1)\cdot P(Z=\widetilde{z}_1)
    \mathop{+} \\
    & \qquad\qquad \sum_{s_1\in S} \delta(s_0)(\epsilon,s_1)\cdot
    \sum_{\substack{\widetilde{z}_1\in\Omega^n_{s_1}\\Y(\widetilde{z}_1)=\widetilde{y}}}
    X_\trans(\widetilde{z}_1)\cdot P(Z=\widetilde{z}_1) \mathop{+} \\
    & \qquad\qquad \sum_{s_1\in S}[s_0=s\land a_1=a\land
    s_1=s']\cdot\delta(s_0)(a_1,s_1)\cdot
    \sum_{\substack{\widetilde{z}_1\in
        \Omega^{n-1}_{s_1}\\Y(\widetilde{z}_1)=\widetilde{y}_1}}
    P(Z=\widetilde{z}_1) \\
    &= \sum_{s_1\in S} \delta(s_0)(a_1,s_1) \cdot
    C^{\widetilde{y}_1}_{s_1,\trans} + \sum_{s_1\in S}
    \delta(s_0)(\epsilon,s_1) \cdot
    C^{\widetilde{y}}_{s_1,\trans} \mathop{+} \\
    & \qquad\qquad [s_0=s \land a=a_1]\cdot \delta(s)(a,s')\cdot
    P_{s'}(Y=\widetilde{y}_1).
  \end{align*}
  \qed
\end{proof}

\todo[inline]{B: Move example to the end of subsection?}

Since the equations are contractive after some iterations (cf.\
Appendix~\ref{sec:contractivity}),
they have a unique fixpoint, which can be
approximated by (Kleene) iteration or computed via a solver.
For this we have to be able to determine $P_{s_0}(Y=\widetilde{y})$
for $\widetilde{y} = a_1\widetilde{y}_1$, which can be done with a
similar fixpoint equation (adapt the proof of
Prop.~\ref{prop:cond-exp} to the case where $X_\trans$ is the constant
$1$-function): $P_{s_0}(Y=\epsilon) = 1$ and otherwise:
\[ P_{s_0}(Y=\widetilde{y}) = \sum_{s_1\in S}
  \delta(s_0)(a_1,s_1) \cdot P_{s_1}(Y=\widetilde{y}_1) \mathop{+}
  \sum_{s_1\in S} \delta(s_0)(\epsilon,s_1) \cdot
  P_{s_1}(Y=\widetilde{y}). \]

\vspace*{-1cm}
\begin{adjustbox}{valign=C,raise=\strutheight,minipage={1\linewidth}}
  \begin{wrapfigure}[7]{r}{0.16\linewidth} 
    \begin{tikzpicture}
        \node[state]             (S0) {$s_0$};
        \draw[every loop]
      
        (S0) edge[loop above]  node {$\subalign{\epsilon \quad &\sfrac{1}{2}\\
        \alpha \quad &\sfrac{1}{4} \\
        \beta \quad &\sfrac{1}{4}\\}$} (S0);
      
        \draw[->, black] (-1,0) -- (S0) node[pos=.35]{};
      \end{tikzpicture}
  \end{wrapfigure}%
  \strut{}
  \vspace*{0.5cm} 
  \begin{example}
    Given the following HMM on the right where the states and
    transitions are known, but the probabilities have to be adjusted
    by observing the system. The three transitions are
    named $\trans_1 = (s_0, \epsilon, s_0)$,
    $\trans_2 = (s_0, \alpha, s_0)$, $\trans_3 = (s_0, \beta, s_0)$
    and the observation sequence is $\widetilde{y}=\alpha$. Then:
    \begin{align*}
      C^{\widetilde{y}}_{s_0,\trans_1} &=
      \underbrace{\delta(s_0)(\alpha, s_0)}_{\sfrac{1}{4}} \cdot
      \underbrace{C^{\epsilon}_{s_0,\trans_1}}_{0} +
      \underbrace{\delta(s_0)(\epsilon, s_0)}_{\sfrac{1}{2}} \cdot
      C^{\widetilde{y}}_{s_0,\trans_1} +
      1 \cdot \underbrace{\delta(s_0)(\epsilon, s_0)}_{\sfrac{1}{2}} \cdot \underbrace{P_{s_0}(Y = \widetilde{y})}_{\sfrac{1}{4} + \sfrac{1}{2} \cdot P_{s_0}(Y = \widetilde{y})}\\
      &\Rightarrow C^{\widetilde{y}}_{s_0,\trans_1} = \sfrac{1}{2}
      = E_{s_0}[X_{\trans_1} \mid Y = \widetilde{y}] \cdot
      \underbrace{P_{s_0}(Y = \widetilde{y})}_{\sfrac{1}{2}} \quad
      \Rightarrow \, E_{s_0}[X_{\trans_1} \mid Y = \widetilde{y}] =
      1
    \end{align*}
    Similarly, we compute $E_{s_0}[X_{\trans_2} \mid Y = \widetilde{y}] = 1$ and $E_{s_0}[X_{\trans_3} \mid Y = \widetilde{y}] = 0$.
    The adjusted and normalized probability parameters are:
    \begin{center}
      $\delta(s_0)(\epsilon, s_0) = \sfrac{1}{2} \qquad
      \delta(s_0)(\alpha, s_0) =\sfrac{1}{2} \qquad
      \delta(s_0)(\beta, s_0) = 0$
    \end{center}
    In practice one will of course make longer or multiple
    observations before adjusting the parameters.
  \end{example}%
\end{adjustbox}



  



\subsection{Using the EM Algorithm}

We will now introduce the so-called Expectation Maximization (EM)
Algorithm \cite{dlr:em-algorithm}, which is commonly used to derive
the 
Baum-Welch algorithm
\cite{r:tutorial-hidden-markov-models,b:gentle-tutorial-EM} for
parameter estimation. 
\todo{\textbf{Re}: Wikipedia: The Baum-Welch algorithm is a special
  case of the EM algorithm

Rabiner 89:
- elaborates on the connection between Baum Welch (forward-backward algorithm) and EM algorithms
- shows some kind of "proof" (beginning) of P maximizing

Bimes 98:
- using the Q function from the EM algorithm to derive Baum Welch}
It explains how to suitably adjust
(probabilistic) parameters of a system in such a way that the
likelihood of observing the given output of the system increases.
We assume that the higher the
probability for observed sequences, the closer the parameters are to
their actual values.
This procedure is divided into two phases:
the Expectation and 
Maximization phase.

Fix 
an HMM with 
known 
(graph structure) and 
unknown parameters (transition
probabilities). 
The unknown parameters, denoted by $\theta$, can be
learned by observing the system.  We will use $\theta$ in conditional probabilities
or expectations to clarify the parameter dependency. E.g.,
$\delta(\trans\mid \theta)$ with $\trans = (s,a,s')$ stands for
$\delta(s)(a,s')$ under the parameter setting~$\theta$.

The algorithm works iteratively in two phases. $\theta^{t}$
always denotes our current best guess of the probabilistic parameters,
$\theta$ denotes the new parameters that we wish to learn and improve
iteratively given an observation sequence $\widetilde{y}$.  In the
first phase we calculate $Q(\theta \mid \theta^{t})$ denoting the
expected value of the log likelihood function for $\theta$ with
respect to the current conditional probability of $Z$ given an
observation and the current estimates of the parameter
$\theta^t$. More concretely:
\[ Q(\theta \mid \theta^{t}) = E_{Z \mid Y, \theta^{t}} [\log P(Y, Z
  \mid \theta)], \] which denotes the expectation of the random
variable $\widetilde{z}\mapsto \log P(Y=\widetilde{y},Z=\widetilde{z}\mid \theta)$
in an updated probability space where the probability function is
$P'(\widetilde{z}) = P(Z=\widetilde{z}\mid Y=\widetilde{y},\theta^{t})$. Here it
is understood that $\log 0 = -\infty$ and $0\cdot (-\infty) = 0$.



After the first phase follows the Maximization phase, where
$\theta^{t+1}$ is determined as
$\argmax_{\theta} Q(\theta \mid \theta^{t})$ and the algorithm
subsequently starts again with phase one. This happens iteratively
until $\theta^{t+1}=\theta^{t}$ or the improvements are below some
threshold. In general we will converge to a local optimum, finding the
global optimum is typically infeasible.
The guarantee of the EM algorithm is that
$P(Y=\widetilde{y}\mid \theta) > P(Y=\widetilde{y}\mid \theta^t)$ whenever
$Q(\theta\mid \theta^t) > Q(\theta^t\mid \theta^t)$.

\begin{theoremrep}
  In our setting it holds that
  \[Q(\theta \mid \theta^{t}) = \sum_{s \in S} \sum_{\substack{\trans
        = (s, a, s') \in T_\delta}} \log \delta(\trans \mid \theta)
    \cdot E_{s_0}[X_\trans \mid Y = \widetilde{y}, \theta^{t}]. \] The value
  $Q(\theta \mid \theta^{t})$ is maximal when the parameters $\theta$
  are as follows: for every transition $\trans$ we set
  $\delta(\trans\mid \theta)$ proportional to
  $E_{s_0}[X_{\trans} \mid Y=\widetilde{y}, \theta^{t}]$.
\end{theoremrep}

\begin{proof}
  In this proof we write $\trans_i(\widetilde{z})$ for the $i$-th
  transition of $\widetilde{z}$, i.e., if
  $\widetilde{z} = s_0a_1s_1a_2\dots a_ms_m$, then
  $\trans_i(\widetilde{z}) = (s_{i-1},a_i,s_i)$. Note that
  $1\le i\le \sfrac{(|\widetilde{z}|-1)}{2}$. Hence
  $P(Z = \widetilde{z} \mid \theta) =
  \prod_{i=1}^{\sfrac{(|\widetilde{z}|-1)}{2}}
  \delta(\trans_i(\widetilde{z}) \mid \theta)$.

  Furthermore:
  \begin{align*}
    Q(\theta \mid \theta^{t}) &= E_{Z \mid Y, \theta^{t}} [\log P(Y, Z \mid \theta)]\\
    &= \sum_{\widetilde{z}\in\Omega^n_{s_0}} \log P(Y = \widetilde{y}, Z = \widetilde{z} \mid \theta) \cdot P(Z = \widetilde{z} \mid Y = \widetilde{y}, \theta^{t})\\
    &=
    \sum_{\substack{\widetilde{z}\in\Omega^n_{s_0}\\Y(\widetilde{z})=\widetilde{y}}}
    \log P(Z = \widetilde{z} \mid \theta) \cdot P(Z = \widetilde{z}
    \mid Y = \widetilde{y}, \theta^{t})\\
    &= \sum_{\substack{\widetilde{z}\in\Omega^n_{s_0}\\Y(\widetilde{z})=\widetilde{y}}} \log \prod_{i=1}^{\sfrac{(|\widetilde{z}|-1)}{2}} \delta(\trans_i(\widetilde{z}) \mid \theta) \cdot P(Z = \widetilde{z} \mid Y = \widetilde{y}, \theta^{t})\\
    &=
    \sum_{\substack{\widetilde{z}\in\Omega^n_{s_0}\\Y(\widetilde{z})=\widetilde{y}}}
    \sum_{i=1}^{\sfrac{(|\widetilde{z}|-1)}{2}} \log \delta(\trans_i(\widetilde{z}) \mid \theta) \cdot P(Z = \widetilde{z} \mid Y = \widetilde{y}, \theta^{t})\\
    &= \sum_{\trans \in T_\delta} \sum_{\substack{\widetilde{z}\in\Omega^n_{s_0}\\Y(\widetilde{z})=\widetilde{y}}} X_\trans(\widetilde{z}) \cdot \log \delta(\trans \mid \theta) \cdot P(Z = \widetilde{z} \mid Y = \widetilde{y}, \theta^{t})\\
    &= \sum_{\trans \in T_\delta} \log \delta(\trans \mid \theta) \cdot
    \underbrace{\sum_{\substack{\widetilde{z}\in\Omega^n_{s_0}\\Y(\widetilde{z})=\widetilde{y}}}
      X_\trans(\widetilde{z}) \cdot P(Z = \widetilde{z} \mid Y =
      \widetilde{y}, \theta^{t})}_
    {E_{s_0}[X_\trans \mid Y = \widetilde{y}, \theta^{t}]}\\
    &= \sum_{s \in S} \sum_{\trans \in
        T^s_\delta} \log \delta(\trans \mid \theta) \cdot
    E_{s_0}[X_\trans \mid Y = \widetilde{y}, \theta^{t}]
  \end{align*}
  The third-last equality is given by the following computation, where
  $A \subseteq B^*$ and $a_i$ denotes the $i$-th symbol of
  $a\in B^*$. In this case nested sums can be rewritten as
  follows, where $\#_b(a)$ stands for the number of occurrences of $b$
  in $a$:
  \begin{align*}
    \sum_{a \in A}\sum_{i=1}^{|a|} T_{a_i}\cdot S_a &= \sum_{a \in
      A}S_a \cdot \sum_{i=1}^{|a|} T_{a_i} = \sum_{a \in A} S_a \cdot
    \sum_{b \in B} \underbrace{\sum_{\substack{1\le i\le
          |a|\\a_i=b}} \overbrace{T_{a_i}}^{T_b}}_{\#_b(a) 
      \cdot T_b} \\
    &= \sum_{b \in B}\sum_{a \in A} \#_b(a) \cdot T_b\cdot S_a
  \end{align*}
  The last value in the computation above can be maximized for each
  $s\in S$ independently and we obtain:
  \begin{align*}
    &\sum_{\trans \in T^s_\delta} \log \underbrace{\delta(\trans \mid \theta)}_{p_i} \cdot \underbrace{E_{s_0}[X_\trans \mid Y = \widetilde{y}, \theta^{t}]}_{a_i} = \sum_i (\log p_i)\cdot a_i
  \end{align*}
  This value is maximal if $p_i = \frac{a_i}{\sum_i a_i}$, which
  concludes the proof. This is a consequence of Gibbs' inequality,
  which says that, given a probability function $p\colon I\to [0,1]$
  with $I$ finite, then for every other probability function
  $q\colon I\to [0,1]$, we have that
  \[ \sum_{i\in I} p_i\log p_i \ge \sum_{i\in I} p_i\log q_i. \]

  It is also related to the fact that Kullback-Leibler divergence is
  always non-negative. Since the result is easy to derive we prove it
  in Lemma~\ref{lem:EM} 
    \qed
\end{proof}

\begin{toappendix}

\begin{lemma}
  \label{lem:EM}
  Let $a_i \ge 0$, $i\in \{1,\dots,m\}$, be fixed. Let $p_i\in [0,1]$
  be unknown values such that $\sum_{i = 1}^m p_i = 1$. Then
  \[ \sum_{i = 1}^m a_i\cdot \log p_i \]
  is maximal if $p_i = \frac{a_i}{\sum_{i=1}^m a_i}$.
\end{lemma}

\begin{proof}
  We can assume that $\log = \ln$, since logarithms differ only by a
  constant factor. Furthermore we can see that in order to achieve the
  maximal value, $p_i$ must be strictly larger than $0$ whenever
  $a_i>0$ and $p_i=0$ otherwise. (Remember the convention that
  $0\cdot \log 0 = 0\cdot(-\infty) = 0$.) Hence we can assume
  without loss of generality that $a_i > 0$ for all $i$.
  
  Since $\sum_{i = 1}^m p_i = 1$ we can replace $p_m$ by
  $1-\sum_{i=1}^{m-1} p_i$ and obtain
  \begin{align*}
    &\sum_{i = 1}^{m-1} a_i\cdot \ln p_i + a_m\cdot
    \ln \big(1-\sum_{i=1}^{m-1} p_i\big) 
  \end{align*}
  Remembering that $\frac{d}{dx} \ln x = \frac{1}{x}$, we now compute
  the partial derivates with respect to $p_j$ where $j \neq m$.
  \begin{align*}
    &\frac{\partial}{\partial p_j} \Big( \sum_{i = 1}^{m-1} a_i\cdot \ln
    p_i + a_m\cdot
    \ln \big(1-\sum_{i=1}^{m-1} p_i\big) \Big) \\
    &= \frac{a_j}{p_j} + a_m\cdot \frac{1}{1-\sum_{i=1}^{m-1}
      p_i}\cdot (-1) \\
    &= \frac{a_j}{p_j} - \frac{a_m}{p_m} = \frac{a_j\cdot p_m -
      p_j\cdot a_m}{p_j\cdot p_m}
  \end{align*}
  This equals $0$ if $a_j\cdot p_m - p_j\cdot a_m = 0$. We sum up over
  all indices $j$ and get
  \[ 0 = \sum_{j=1}^m \left(a_j \cdot p_m - p_j \cdot a_m \right)
    = p_m \cdot \sum_{j=1}^m a_j - a_m \cdot \sum_{j=1}^m p_j
    = p_m \cdot \sum_{j=1}^m a_j - a_m\]
  This implies \[ p_m = \frac{a_m}{\sum_j a_j}\] and by substitution,
  we obtain an analogous formula for all $p_j$. We can check that all
  conditions $a_j\cdot p_m - p_j\cdot a_m = 0$ are satisfied.

  The maximum must be reached in the point where all derivatives
  are zero, from which the statement follows. \qed
\end{proof}

\end{toappendix}

Note that there might be states where all outgoing transitions have
conditional expectation zero, i.e., such a state can not be reached via
the observation sequence. In this case we keep the previous
parameters. If we adhere to this, we can always guarantee that the
requirement of Lemma~\ref{prop:prob-space-welldefined} is maintained,
since if an outgoing transition of a state has conditional expectation
greater than zero, there must be a path of non-zero probability to an
observation.




\begin{toappendix}

\section{Contractivity}
\label{sec:contractivity}

The claims on contractivity made in the paper deserve further
elaboration. We first define the notion of a contractive function. 

\begin{definition}[Contractive function]
  Let $\mathbb{R}^\mathcal{W}$ be the set of all functions from a set
  $\mathcal{W}$ to $\mathbb{R}$. We use the supremum (or maximum) distance and
  define
  $d_\mathrm{sup}(g_1,g_2) = \sup_{W\in\mathcal{W}} |g_1(W)-g_2(W)|$
  for $g_1,g_2\colon \mathcal{W}\to \mathbb{R}$.

  A function $F\colon \mathbb{R}^\mathcal{W}\to \mathbb{R}^\mathcal{W}$ is
  contractive whenever for all $g_1,g_2\in \mathbb{R}^\mathcal{W}$ it holds
  that $d_\mathrm{sup}(F(g_1),F(g_2)) \le q\cdot d_\mathrm{sup}(g_1,g_2)$ for some $0\le q < 1$.

  We say that $F$ is contractive after $k$ iterations if $F^k$ is
  contractive.
\end{definition}

It is well-known from the Banach fixpoint theorem that contractive
functions over complete metric spaces have unique fixpoints and
$\mathbb{R}^\mathcal{W}$ with the sup-metric is complete. Furthermore
any sequence $(g_i)_{i\in\mathbb{N}}$ with $g_{i+1} = F(g_i)$
converges to this fixpoint. Now, every fixpoint of $F$ is a fixpoint
of $F^k$ and vice versa. The latter direction holds, since fixpoint
iteration for $F$ from a fixpoint $x$ of $F^k$ (with $F^k(x) = x$)
clearly converges again to $x$. Hence a function $F$ that is
contractive after $k$ iterations also has a unique fixpoint and enjoys
the same convergence property (although with a potentially slower
convergence rate).

\smallskip

We now argue why the fixpoint functions that we consider are
contractive after a certain number of iterations.

The fixpoint equation systems set up in
Sct.~\ref{sec:probability-space-random-variables} (proof of
Prop.~\ref{prop:prob-space-welldefined}) and Sct.~\ref{prop:cond-exp}
are over a set $\mathcal{W}$ of variables of the form
$W_s^{\widetilde{y}}$, where $s\in S$ and $\widetilde{y}\in\Sigma^*$
is the suffix of a given word $\bar{y}\in\Sigma^*$ with
$n=|\bar{y}|$. (Or alternatively the variables are of the form
$S_s^n$, see the proof of Prop.~\ref{prop:prob-space-welldefined},
leading to an analogous argument.) We define
$o(W_s^{\widetilde{y}}) = \widetilde{y}$. Note that $\mathcal{W}$ is
finite, since the state space $S$ is finite.

The corresponding fixpoint function is a monotone function
$F\colon \mathbb{R}^\mathcal{W}\to \mathbb{R}^\mathcal{W}$ where, for
$g\colon \mathcal{W}\to \mathbb{R}$:
\begin{equation}
 F(g)(W) = \sum_{W'\in\mathcal{W}} p_{W,W'} \cdot g(W') +
 D_W, \label{eq:fixpoint}
\end{equation}
where $p_{W,W'} \in [0,1]$ such that for each $W\in\mathcal{W}$ we
have $\sum_{W'\in\mathcal{W}} p_{W,W'} \le 1$ and $D_W$ is a
non-negative constant. Furthermore $p_{W,W'} > 0$ implies
$|o(W)| \ge |o(W')|$. In addition we can assume that $o(W)=\epsilon$
implies $p_{W,W'} = 0$ (and hence $F(g)(W) = D_W$ is a constant).

Such functions are clearly non-expansive (which means that the
contractivity requirement holds for $q=1$) but not necessarily
contractive.

However, we know that the probabilities $p_{W,W'}$ are transition
probabilities and the length of the observed sequence decreases if one
takes a transition that is labelled with an observable symbol. Due to
the requirement of Prop.~\ref{prop:prob-space-welldefined} we know
that each state has a path of non-zero probability that contains such
an observation. For each state $s\in S$ we consider the minimum length
of such a path and we take the maximum over all these minimums and
obtain $k$. Then we know that the fixpoint equation associated with
$F^k$ is of the same form as for $F$ above (see~(\ref{eq:fixpoint}))
and additionally for each $W\in \mathcal{W}$
\begin{itemize}
\item there exists $\overline{W}\in\mathcal{W}$ with
  $p_{W,\overline{W}} > 0$ and $|o(W)| > |o(\overline{W})|$ (if we
  take a transition with the next label to observe, reducing the
  length of the observation sequence) \emph{or}
\item $\sum_{W'\in\mathcal{W}} p_{W,W'} < 1$ (if a state has an
  outgoing transition with a label that does not match the next
  observation). 
\end{itemize}
This means that either the second condition holds after at most
$n\cdot k$ iterations or we reach the last observation of $\bar{y}$ on
a path of non-zero probability of length at most $n\cdot k$. The
latter means that the term for $F^{n\cdot k}(g)(W)$ contains --
multiplied with a non-zero probability -- a variable $\overline{W}$
with $o(\overline{W}) = \epsilon$, for which $F(g)(\overline{W})$ is
constant. That is, after $m+1$ iterations the corresponding fixpoint
equation~(\ref{eq:fixpoint}) satisfies
$q_W := \sum_{W'\in\mathcal{W}} p_{W,W'} < 1$ for each
$W\in\mathcal{W}$. Then we have, given
$g_1,g_2\colon \mathcal{W}\to \mathbb{R}$:

\begin{align*}
  &d_\mathrm{sup}(F^{m+1}(g_1),F^{m+1}(g_2)) \\
  &= \max_{W\in\mathcal{W}} \Big| \big(\sum_{W'\in\mathcal{W}}
  p_{W,W'} \cdot g_1(W') + D_W\big) -
  \big(\sum_{W'\in\mathcal{W}} p_{W,W'} \cdot g_2(W') + D_W\big)\Big| \\
  &\le \max_{W\in\mathcal{W}}
  \sum_{W'\in\mathcal{W}} p_{W,W'} \cdot |g_1(W')-g_2(W')| \\
  &\le \max_{W\in\mathcal{W}} \underbrace{\big(\sum_{W'\in\mathcal{W}}
    p_{W,W'}\big)}_{q_W}\cdot
  \max_{W'\in\mathcal{W}} |g_1(W')-g_2(W')| \\
  &\le \underbrace{(\max_{W\in\mathcal{W}} q_W)}_q \cdot
  \max_{W'\in\mathcal{W}} |g_1(W')-g_2(W')| = q\cdot
  d_\mathrm{sup}(g_1,g_2)
\end{align*}

The first inequality uses the fact that
$|\sum_{i\in I} a_i| \le \sum_{i\in I} |a_i|$, while the second
inequality holds since
$\sum_{i\in I} p_i\cdot a_i \le (\sum_{i\in I} p_i)\cdot \max_{i\in I}
a_i$.

Since $q<1$ we have contractivity after $m+1$ iterations.

In fact, the arguments are similar to the setting of absorbing
Markov chains \cite{gs:markov-chains}, where the absorption property
is used to guarantee unique solutions.

\end{toappendix}

\section{Conclusion}
\label{sec:conclusion}

In this paper, we considered HMMs that admit unobservable $\epsilon$-transitions. 
We presented algorithms for determining the most probable explanation 
(i.e. a sequence of hidden states) given an observation and a method for 
parameter learning. 
For this, we generalized the Viterbi and the Baum-Welch algorithm to consider 
$\epsilon$-transitions (including $\epsilon$-loops) and provided the respective proofs of their soundness. 
By allowing state changes of which the observer is unaware we can model false 
negatives, i.e. actions that have taken place but have not been observed by a sensor. 
This extends the applicability of HMMs as a modeling technique to the domain of sensor-based systems, 
which always have to consider the probability of sensor errors. 
For example, we now have the methods to compare observations made by sensors with the computed most likely explanation. 
When these two drift further apart over time, we can conclude that the real-world system 
is subject to parameter drift or degrading sensor quality. 
Furthermore, we plan to use the HMMs to clean data sets by replacing
observations with their most probable explanation. Parameter learning
will be needed to learn and adapt the model parameters based on recorded observations.

\bibliographystyle{plain}
\bibliography{references}

\end{document}